\def\year{2020}\relax
\documentclass[letterpaper]{article} % DO NOT CHANGE THIS
\usepackage{aaai20}  % DO NOT CHANGE THIS
\usepackage{times}  % DO NOT CHANGE THIS
\usepackage{helvet} % DO NOT CHANGE THIS
\usepackage{courier}  % DO NOT CHANGE THIS
\usepackage[hyphens]{url}  % DO NOT CHANGE THIS
\usepackage{graphicx} % DO NOT CHANGE THIS
\usepackage{graphicx}  % DO NOT CHANGE THIS
\usepackage{booktabs}       
\usepackage{amsfonts, amsmath, amssymb, amsthm}
\usepackage{nicefrac}       
\usepackage{microtype}   
\usepackage{graphicx}
\usepackage{subfig}
\usepackage{xcolor}
\usepackage{algorithm, algpseudocode}

\newtheorem{proposition}{Proposition}

\newtheorem{problem}{Problem}

\newcommand{\cX}{{\mathcal X}}
\newcommand{\rE}{{\mathbf{E}}}

\title{Stable Learning via Sample Reweighting}
\author{Zheyan Shen,\textsuperscript{\rm 1}
		Peng Cui,\textsuperscript{\rm 1}
		Tong Zhang,\textsuperscript{\rm 2}
		Kun Kuang\textsuperscript{\rm 1, 3}\\
		\textsuperscript{\rm 1}Tsinghua University,
		\textsuperscript{\rm 2}The Hong Kong University of Science and Technology\\
		\textsuperscript{\rm 3}Zhejiang University\\
		shenzy17@mails.tsinghua.edu.cn, cuip@tsinghua.edu.cn\\
		tongzhang@tongzhang-ml.org, kkun2010@gmail.com
}

\begin{document}

\maketitle

\begin{abstract}
We consider the problem of learning linear prediction models with model misspecification bias.
In such case, the collinearity among input variables may inflate the error of parameter estimation, resulting in instability of prediction results when training and test distributions do not match.
In this paper we theoretically analyze this fundamental problem and propose a sample reweighting method that reduces collinearity among input variables. 
Our method can be seen as a pretreatment of data to improve the condition of design matrix, and it can then be combined with any standard learning method for parameter estimation and variable selection.
Empirical studies on both simulation and real datasets demonstrate the effectiveness of our method in terms of more stable performance across different distributed data.
\end{abstract}

\section{Introduction}  
We consider the classical problem of predicting a target $y$ using a linear combination of $p$ input variables $x=[x_1,\ldots,x_p] \in \mathbb{R}^p$.
In practice many machine learning methods can be used for such purpose. 
However, the performance of most machine learning methods deteriorate when the distribution of the test data deviates from that of the training data.
This is because the traditional learning methods rely on a fundamental assumption that the data drawn at training time are from the same underlying distribution as the test data.
In many real situations, however, this assumption can be violated since we have no prior knowledge on the test data which will be generated in the future.
Therefore, a large bunch of learning methods which assume the availability of the test data distribution (e.g. transfer learning \cite{pan2010survey}) are not readily applicable at such scenarios.

In this paper, we consider the {\em stable learning} problem that directly addresses this fundamental issue.
The goal of stable learning is to learn a predictive model that performs uniformly well on any data point $x$.
We actually need two assumptions: 
(1) There exists a stable structure between target $y$ and predictor $x_p$ which keeps invariant across the whole distribution. 
(2) There also exist spurious associations brought by external biases which could be unstable across different environments. 
It is common in practice that, due the different time spans, regions and strategies we collect the data, there usually exist such spurious associations. 
If we only leverage the stable structure for prediction, we can ensure good prediction performance even when the unknown test distribution significantly differs from the training distribution. 

The main challenge of stable learning is that in real applications, we can not expect to choose a completely correct model for the underlying application problem.
We show in this paper that if an incorrect model is used at the training time (which is inevitable in practice), the existence of collinearity among variables (i.e. linear dependence between two or more input variables) can inflate a small misspecification error arbitrarily large, thus causes instability of prediction performance across different distributed test data.
Therefore, how to reduce collinearity is of paramount importance in the stable learning problem.

Collinearity \cite{Alin2010Multicollinearity,farrar1967multicollinearity} can also be regarded as an ill-conditioning \cite{Fildes1993Conditioning} or lack of orthogonality for the design matrix $\mathbf{X}$.
It brings challenges to evaluate the individual importance of variables in a linear model since their contributions are interchangeable.
As a long standing problem in statistics, considerable efforts have been made on collinearity.
The major way to handle collinearity is performing variable selection.  
Mutual information based methods like \cite{kononenko1994estimating,raileanu2004theoretical,ding2005minimum,peng2005feature} can be seen as a pretreatment of data. 
They basically chooses a subset of variables that are representative and discriminative for response variable by maximizing the correlation between selected variables and response while minimizing the correlation among selected variables.
Another strand of methods are based on regularization techniques, which get significant attention because they simultaneously achieve good performance on parameter estimation and variable selection.
Distinguished by different assumptions on the variables' structure, the approach of \cite{zou2005regularization,lorbert2010exploiting,grave2011trace} designs penalty terms by constraining the correlated variables to be either all selected or not selected at all as "clusters", and the approach of \cite{chen2013uncorrelated,takada2018IILasso,zhou2010exclusive} chooses only one variable within a single cluster.
The performances of these methods depend highly on the correct hypothesis on the structure of variables.
When there are inactive variables or multiple active variables in the correlated clusters, these methods would suffer from either loss of information or inclusion of inactive variables, resulting in unstable behaviors over changing distributions.

In this paper, we focus on the stable learning problem of linear models under model misspecification.
We first provide a theoretical analysis on the worst estimation error brought by misspecification bias and demonstrate its direct connection to collinearity.
In order to alleviate the collinearity among variables, we propose a novel sample reweighting scheme.
We theoretically prove that there exist a set of sample weights which can make the design matrix near orthogonal in an idealized situation. 
Accordingly, we propose a Sample Reweighted Decorrelation Operator (SRDO) to reduce collinearity in practice.
Specifically, we construct an uncorrelated design matrix $\tilde{\mathbf{X}}$ from original $\mathbf{X}$ as the 'oracle', and learn the sample weights $w(x)$ by estimating the density ratio of underlying uncorrelated distribution $\tilde{D}$ and original distribution $D$. 

This method can be regarded as a general data pretreatment method that improves the condition of the underlying design matrix for prediction purpose.
The learned sample weights can be easily integrated into standard linear regression methods such as ordinary least squares regression, Lasso and logistic regression for classification task to improve their stability across different distributed test data.

The main contributions of our paper are as follows:
\begin{itemize}
	\item We investigate the stable learning problem of linear models with model misspecification under changing test distributions. 
	This problem is fundamental and of paramount importance to real applications which require model robustness and stability.
	We do not assume the availability of the test data distributuion, which is more realistic at practice.
	\item We theoretically prove the direct connection of prediction stability and the collinearity between variables, and propose a novel Sample Reweighted Decorrelation Operator (SRDO) to reduce the collinearity of design matrix.
	\item SRDO is a general data pretreatment method that can be easily integrated into a wide range of classical methods for parameter estimation, variable selection and prediction, and the extensive experiments on both synthetic and real datasets demonstrate its superior performances in both prediction stability and accuracy under changing distributions.
\end{itemize}

\section{Problem and Method}

\textbf{Notations.} 
In this paper, we let $n$ denote the sample size, $p$ denote the dimension of observed variables. 
For any matrix $\mathbf{A}\in \mathbb{R}^{n\times p}$, let $\mathbf{A}_{i,}$ and $\mathbf{A}_{,j}$ represent the $i^{th}$ row and the $j^{th}$ column in $\mathbf{A}$, respectively. 
For any vector $\mathbf{v} = (v_1,v_2,\cdots,v_m)^\top$, let $\|\mathbf{v}\|_1 = \sum_{i=1}^m|v_i|$ and $\|\mathbf{v}\|_2^2 = \sum_{i=1}^m v_i^2$.

We first define the stable learning problems as follows:
\begin{problem}
	(\textbf{Stable Learning}) : Given the target $y$ and $p$ input variables $x=[x_1,\ldots,x_p] \in \mathbb{R}^p$, the task is to learn a predictive model which can achieve \textbf{uniformly} small error on \textbf{any} data point.
	\label{prob:stable_learning}
\end{problem}
In this paper, we study the above problem in the scope of linear models for regression and classification.

\subsection{Stable Linear Models for Regression}
We consider the linear regression problem with model misspecification.
Specifically, we can assume the target $y$ is generated by following from:
\begin{equation}
y=x^{\top} \overline{\beta}_{1 : p}+\overline{\beta}_{0}+b(x)+\epsilon,
\label{equ:Y_generation}
\end{equation}
where $x\in \mathbb{R}^{p}$ is input vector, $b(x)$ is a bias term that depends on $x$, such that $|b(x)| \leq \delta$, and $\epsilon$ is zero-mean noise with variance $\sigma^2$.

In stable learning, we assume the linear part of generation model is stable and invariant to unknown distribution shift while the misspecification bias $b(x)$ could be unstable.
In such sense, we want to estimate $\overline{\beta}$ as accurately as possible. 
Along with the property that the bias term $b(x)$ is uniformly small for all $x$, we can make reliable prediction for all $x$. 
In particular, a change of distribution does not matter for prediction purpose.

Given training data $\left\{\left(x_{1}, y_{1}\right), \ldots,\left(x_{n}, y_{n}\right)\right\}$, where the design matrix $\mathbf{X}$ is drawn from a distribution $D$.
We assume that $\left\|x_{i}\right\|_{2} \leq 1$. 
The standard approach of least squares regression solves the following problem:
\begin{equation}
\hat{\beta}=\arg \min _{\beta} \sum_{i=1}^{n}\left(x_{i}^{\top} \beta_{1 : p}+\beta_{0}-y_{i}\right)^{2}
\label{equ:OLS}
\end{equation}
Let $\gamma^{2}$ be the smallest eigenvalue of the centered covariance matrix $n^{-1} \sum_{i}\left(x_{i}-\overline{x}\right)\left(x_{i}-\right.\overline{x} )^{\top}$, where $\overline{x}=n^{-1} \sum_{i} x_{i}$. 

The approach considered in this paper is
motivated by the following theoretical result, which shows the effect of model mis-specification bias even when the sample size is infinity. 

\begin{proposition}
	Consider the least squares solution when the sample size is infinity:
	\begin{equation}
	\hat{\beta}=\arg \min _{\beta} \mathbf{E}_{(x, y)}\left(x^{\top} \beta_{1 : p}+\beta_{0}-y\right)^{2} .
	\label{equ:OLS_inf}
	\end{equation}
	The estimation bias caused by the worst case perturbation error $|b(x)| \leq \delta$ can be as bad as $\|\hat{\beta}-\overline{\beta}\|_2 \leq 2 (\delta / \gamma) +\delta$, where $\gamma^2$ is the smallest eigenvalue of $\mathbf{E}(x-\mathbf{E} x)(x-\mathbf{E} x)^{\top}$.
	\label{theorem:perturbation_bias}
\end{proposition}

\begin{proof}
	Let $\Delta \beta=\beta-\bar{\beta}$ and $\Delta \hat{\beta}=\hat{\beta}-\bar{\beta}$.
	We have
	\begin{equation}
	\Delta\hat{\beta}=\arg\min_{\Delta \beta} \rE_{x} (x^\top \Delta\beta_{1:p}+\Delta\beta_0 -b(x))^2.
	\end{equation}
	At the optimal solution, we have $\Delta\hat{\beta}_0=\rE_x b(x)-\rE_x x^\top \Delta\hat{\beta}_{1:p}$. 
	By eliminating $\beta_0$, and let $\tilde{x}=x-\rE x$, and $\tilde{b}(x)=b(x)-\rE_x b(x)$, we have
	\begin{equation}
	\Delta\hat{\beta}_{1:p}=\arg\min_{\Delta\beta_{1:p}} (\tilde{x}^\top \Delta\beta_{1:p}-\tilde{b}(x))^2 .
	\end{equation}
	It follows that
	\begin{equation}
	\Delta \hat{\beta}_{1:p}= (\rE \tilde{x}\tilde{x}^\top)^{-1} \rE \tilde{b}(x) \tilde{x}.
	\end{equation}
	This implies that
	$\|\Delta \hat{\beta}_{1:p}\|_2 \leq \delta/\gamma$.
	Moreover, it implies that $|\Delta \hat{\beta}_0| \leq \delta + \delta/\gamma$. We thus obtain the desired bound.
\end{proof}
From Proposition \ref{theorem:perturbation_bias}, we observe that the worst case estimation error goes to infinity when $\gamma$ goes to zero. This implies that when the variables are highly collinear, the ordinary least squares method produces a poor solution even when the training data size is very large (or infinity). 

So the problem of stable learning is to find stable $\hat{\beta}$ such that in the infinite sample case, for the worst $|b(x)| \leq \delta$, the estimation error is $\|\hat{\beta}-\overline{\beta}\|_2=O(\delta)$ and independent of $\gamma$. 
This means we tolerate bias caused by collinearity.

To tackle with collinearity, we propose a sample reweighting scheme as follows in the infinite sample case:
\begin{equation}
\hat{\beta}=\arg \min _{\beta} \mathbf{E}_{(x) \sim D} w(x)\left(x^{\top} \beta_{1 : p}+\beta_{0}-y\right)^{2},
\label{equ:WLS_inf}
\end{equation}
where $w(x)$ is the sample weight which is to be learned.

This is equivalent to
\begin{equation}
\hat{\beta}=\arg \min _{\beta} \mathbf{E}_{(x) \sim \tilde{D}}\left(x^{\top} \beta_{1 : p}+\beta_{0}-y\right)^{2},
\label{equ:WLS_transfer}
\end{equation}
where
\begin{equation}
\frac{p_{\tilde{D}}(x)}{p_{D(x)}}=w(x).
\label{equ:density_estimate}
\end{equation}
For $\tilde{D}$ to be a valid distribution, we have $\rE_{x \sim D}[w(x)] = 1$.

The goal of sample reweighting is to improve $\tilde{\gamma}$, where $\tilde{\gamma}^2$ is the smallest eigenvalue of 
\[
\mathbf{E}_{(x)\sim \tilde{D}}(x-\mathbf{E}_{x \sim \tilde{D}} x)(x-\mathbf{E}_{x \sim \tilde{D}} x)^{\top} ,
\]
with $x$ drawn from $\tilde{D}$. 

However, if $\rE_{x \sim D} w(x)^2$ is large, then we have a penalty in the finite sample error caused by the random noise $\epsilon$. 
In fact, in the weighted least squares model, when $n \rightarrow \infty$, by Slutsky's theorem, we have
\begin{equation}
\sqrt{n}(\widehat{\mathbf{\beta}}-\bar{\mathbf{\beta}})\stackrel{d}{ \longrightarrow}N(\mathbf{0},\mathbf{Q}),
\end{equation}
where
\[
\mathbf{Q} = E\left[w(x_{i})\mathbf{x}_{i}\mathbf {x}_{i}^\top\right]^{-1} E\left[w(x_{i})^{2}\mathbf{x}_{i}\mathbf{x}_{i}^\top \epsilon_{i}^2\right] E\left[w(x_{i})\mathbf{x}_{i}\mathbf{x}_{i}^\top\right]^{-1},
\]
then similar to the analysis in Proposition 1, the finite sample estimation error caused by random noise $\epsilon$ is bounded by
\begin{equation}
O\left(n^{-1/2} \sqrt{\rE_{x \sim D} w(x)^{2}} \sigma / \tilde{\gamma}\right) .
\end{equation}
By combining this result with Proposition \ref{theorem:perturbation_bias}, the total estimation error $\|\hat{\beta}-\overline{\beta}\|_2$ (caused by both the bias $b(x)$ and random noise $\epsilon$) in the finite sample case is:
\begin{equation}
O(\delta / \tilde{\gamma})+O\left(n^{-1 / 2} \sqrt{\mathbf{E}_{x \sim {D}} w(x)^{2}} \sigma / \tilde{\gamma}\right) ,
\label{equ:total_bias}
\end{equation}
when $n$ is large. 
The first term on the right hand side is bias, which is independent of the training sample size $n$, and the second term is the square root of the variance, which depends on $n$. 
Reweighting can reduce the bias term, but increases the variance term in general.
Therefore in the small sample case, where $n$ is not large, there is a tradeoff. 

If we can make $\tilde{\gamma}$ close to 1, then the estimation bias brought by $b(x)$ will become $O(\delta / \tilde{\gamma})=O(1)$ as we can assume the misspecification error $\delta$ to be a measurable and bounded "systematic" error and could be seen as a constant value for a specific system.
Thus the total bias is
\begin{equation}
\|\hat{\beta}-\overline{\beta}\|_2=O(1)+O\left(n^{-1 / 2} \sqrt{\mathbf{E}_{x \sim {D}} w(x)^{2}} \sigma \right),
\label{equ:total_bias_W}
\end{equation}
which becomes irrelevant to collinearity and achieves stable prediction we have discussed.

The following proposition shows that under the idealized situation, it is possible to find weights $w$ so that the design matrix becomes near orthogonal (after centering) when the sample size $n \to \infty$.
\begin{proposition}
	Let $p_u(x)$ be the uniform distribution on $\cX=\cX_1 \times \cdots \times \cX_p \subset R^p$, and assume that
	$\rE_{x \sim p_u(x)} \|x\|_2^2 < \infty$.
	Assume that each variable $x_j \in \cX_j$, and the vector $x=[x_j]$ has density $p(x)$ on $\cX$ such that $0 < 2\gamma_0 \leq p_u(x)/p(x) \leq \gamma_1/2$. 
	For all $\xi>0$, and $\zeta>0$, there exists $N$ such that for all $n>N$, with probability larger than $1-\zeta$,
	there exists  and $w$ such that $\|w\|_1=1$,
	$\gamma_0/n \leq \|w\|_\infty \leq \gamma_1/n$, and
	\begin{equation}
	|n^{-1} \sum_i w_i (x_{i,j}-c_j) (x_{i,k}-c_k)| \leq \xi,
	\end{equation}
	where
	$c_j = n^{-1}\sum_i c_{i,j}$ is the mean of each variable $j$ and $j \neq k$.
	\label{pro:exist}
\end{proposition}

\begin{proof}
	Let $w(x)=p_u(x)/p(x)$. 
	For each pair $1 \leq j \neq k \leq p$, we know that for $x = [x^1,\ldots,x^p] \in \cX$,
	\begin{equation}
	\rE_{x \sim p(x)} w(x) (x^k- \rE x^k) (x^j-\rE x^j) =0 .
	\end{equation}
	We also know that
	\[
	\rE_{x \sim p(x)} w(x)=1.
	\]
	Therefore by the law of large numbers,
	there exists $N$ such that with probability larger than $1-\zeta/p^2$,
	when we draw $x_i=[x_{i,1},\ldots,x_{i,p}]$ from $p(x)$ for $i=1,\ldots,n$, we can set $w_i=w(x_i)/\sum_j w(x_j)$, and then
	\begin{equation}
	\left| \frac1n \sum_{i=1}^n w_i (x_{i,j}-c_j) (x_{i,k}-c_k)\right| \leq \xi ,
	\end{equation}
	and 
	\[
	\frac{1}{n}\sum_j w(x_j) \in [0.5,2] .
	\]
	Taking union bound over pairs of $(i,j)$, we obtain the desired result. 
\end{proof}

Assume we standardize all the variables, then the sample covariance matrix becomes correlation matrix $R$ and Proposition \ref{pro:exist} shows that the off-diagonal elements of sample covariance matrix could be bounded arbitrarily small by $\xi$ with the sample weight $w(x)$.

Let $M = R - I_p$, by the Gershgorin circle thorem, we can get $\gamma^2 \geq 1-\|M\|_{\infty} = 1-(p-1)\xi$.
Therefore, by reducing the pairwise correlation between variables (a.k.a. the off-diagonal elements of $R$), we can adjust the smallest eigenvalue to be nearly 1.

Inspired by the Proposition \ref{pro:exist}, we propose a Sample Reweighted Decorrelation Operator (SRDO) to reduce the collinearity of design matrix.
First, we use design matrix $\mathbf{X}$ to generate a column-decorrelated one $\tilde{\mathbf{X}}$ by performing random resampling column-wisely, which breaks down the joint distribution of variables in $\mathbf{X}$ into $p$ independent marginal distributions in $\tilde{\mathbf{X}}$.
Then we can learn the sample weight by density ratio estimation \cite{sugiyama2012density} to transfer the original $\mathbf{X} \sim D$ to $\tilde{\mathbf{X}} \sim \tilde{D}$.

Specifically, we set $\tilde{\mathbf{X}}$ as positive samples ($Z=1$) while $\mathbf{X}$ as negative samples ($Z=0$) and fit a probabilistic classifier.
Via Bayes theorem, density ratio is given by:
\begin{equation}
w(x)
=\frac{p_{\tilde{D}}(\boldsymbol{x})}{p_{D}(\boldsymbol{x})}
=\frac{p(\boldsymbol{x}|\tilde{D})}{p(\boldsymbol{x}| D)}
=\frac{p(\tilde{D})}{p(D)}\frac{p(Z=1|x)}{p(Z=0|x)}.
\label{eq:W_estimate}
\end{equation}
Note that the prior $\frac{p(\tilde{D})}{p(D)}$ is constant over all the samples so we can just omit it. 
To achieve a unit mean of $w(x)$, we can further divide $w(x)$ by its mean $\frac{1}{n}\sum_{i=1}^{n}w(x_i)$.
The algorithm of Sample Reweighted Decorrelation Operator (SRDO) can be summarized as follows:
\begin{algorithm}
	\caption{Sample Reweighted Decorrelation Operator (SRDO)}
	\label{alg:DSRO}
	\begin{algorithmic}[1]
		\Require Design Matrix $\mathbf{X}$
		\State \textbf{for} $i = 1\dots n$ \textbf{do}  
		\State \quad Initialize a new sample $\tilde{x}_i\in \mathbb{R}^{p}$ with empty vector
		\State \quad \textbf{for} $j = 1 \dots p$ \textbf{do}
		\State \quad\quad Draw the $j^{th}$ feature of new sample $\tilde{x}_{i,j}$ from $\mathbf{X}_{,j}$ at random
		\State \quad \textbf{end for}
		\State \textbf{end for}
		% \noindent By now, we have got a column-decorrelated design matrix $\mathbf{\tilde{X}}$ with desired underlying distribution $\tilde{D}$, then we calculate the $w(x)$ via density ratio estimation.
		\State Set $\tilde{x}_i$ as positive samples and $x_i$ as negative samples, then train a binary classifier.
		\State Set $w(x) = \frac{p(Z=1|x)}{p(Z=0|x)}$ for each sample $x_i$ in $\mathbf{X}$, where $p(Z=1|x)$ is the probability of sample $x$ been drawn from $\tilde{D}$ estimated by the trained classifier. 
		\Ensure A set of sample weights $w(x)$ which can deccorelate $\mathbf{X}$
	\end{algorithmic}
\end{algorithm}

\subsection{Stable Linear Models for Classification}
In addition to regression, the idea of sample reweighting can also be applied to classification problems.
For simplicity, we consider the binary classification using logistic regression.

In binary classification, we have $\beta^{\top} x \in R$ and $y \in\{ \pm 1\}$.
The overall loss function is
\begin{equation}
\sum_{i=1}^{n} \ln \left(1+\exp \left(-\beta^{\top} x_{i} y_{i}\right)\right).
\label{equ:logistic_loss}
\end{equation}
Given an approximate solution $\tilde{\beta}$ and let $\tilde{p}_{i}=\tilde{p}\left(x_{i}\right)=1 /\left(1+\exp \left(-\tilde{\beta}^{\top} x_{i}\right)\right),$
we can use Taylor expansion at this solution to approximate the loss function as the following
weighted least squares:
\begin{equation}
\sum_{i=1}^{n} \tilde{p}_{i}\left(1-\tilde{p}_{i}\right)\left(\beta^{\top} x_{i}-z_{i}\right)^{2},
\end{equation}
where $z_i$ is the effective response define by
\begin{equation}
z_i\equiv g(\beta^{\top}x_{i})+(y-\beta^{\top} x_{i}) g^{\prime}(\beta^{\top} x_{i}),
\end{equation}
and $g(x) \equiv \log \frac{x}{1-x}.$

Instead of making the covariance matrix of $\mathbf{X}$ as close as identity, we want the weighted covariance matrix to be decorrelated.
So we can still use the aforementioned methods to estimate $w(x)$ with minor modification as follows:
\begin{equation}
\tilde{p}(x)(1-\tilde{p}(x))w(x)=\frac{p(Z=1|x)}{p(Z=0|x)}.
\end{equation}
In practice, we can ignore those samples which can be predicted accurately by approximate solution with high confidence to reduce the variance of sample weights $w(x)$.
We can then solve a weighted logistic regression as follows:
\begin{equation}
\sum_{i=1}^{n} w\left(x_{i}\right) \ln \left(1+\exp \left(-\beta^{\top} x_{i} y_{i}\right)\right).
\end{equation}

\section{Experiments}
In this section, we evaluate the effectiveness of our algorithm through simulation study and two real world datasets for regression and classification.

\subsection{Baselines}
For regression task, we compare the performance of our method with OLS, Lasso \cite{Tibshirani1996Regression}, Elastic Net \cite{zou2005regularization}, ULasso \cite{chen2013uncorrelated} and IILasso \cite{takada2018IILasso}.
The previous three baselines are classic methods for general purpose, while ULasso and IILasso are specifically designed for tackling collinearity and can be formulated as extensions to Lasso:
\begin{itemize}
	\item \par \noindent Uncorrelated Lasso (ULasso)
	\[
	\min \|Y-\mathbf{X}\beta\|_2^2+\lambda_1\|\beta\|_1+\lambda_2\beta^T\mathbf{C}\beta,
	\]
	where $\mathbf{C} \in \mathbb{R}^{p\times p}$ with each element $C_{jk}=r_{jk}^2$, and $r_{jk} = \frac{1}{n}|\mathbf{X}_{,j}^T\mathbf{X}_{,k}|$.
	
	\item \par \noindent Independently Interpretable Lasso (IILasso)
	\[
	\min \|Y-\mathbf{X}\beta\|_2^2+\lambda_1\|\beta\|_1+\lambda_2|\beta|^T\mathbf{R}|\beta|,
	\] 
	where $\mathbf{R} \in \mathbb{R}^{p\times p}$ with each element $R_{jk}=|r_{jk}|/(1-|r_{jk}|)$, and $r_{jk} = \frac{1}{n}|\mathbf{X}_{,j}^T\mathbf{X}_{,k}|$.
\end{itemize}
For classification task, we substitute log-likelihood loss for square loss in baselines.
The above methods have several hyper-parameters and we tune all the parameters by cross validation.
We apply the SRDO on ordinary least squares in regression tasks and on logistic regression in classification tasks to generate our results.

In our experiments, we consider the case of $n > p$. 
While for the opposite case, one may want to use shrinkage estimators like Ledoit-Wolf \cite{ledoit2004a}.
Due to the limited space, we just show a few settings, complete experiments and implementations will be released soon.

\begin{figure*}[htbp]
	\centering
	\subfloat[Estimation error]{
		\includegraphics[width=0.32\linewidth]{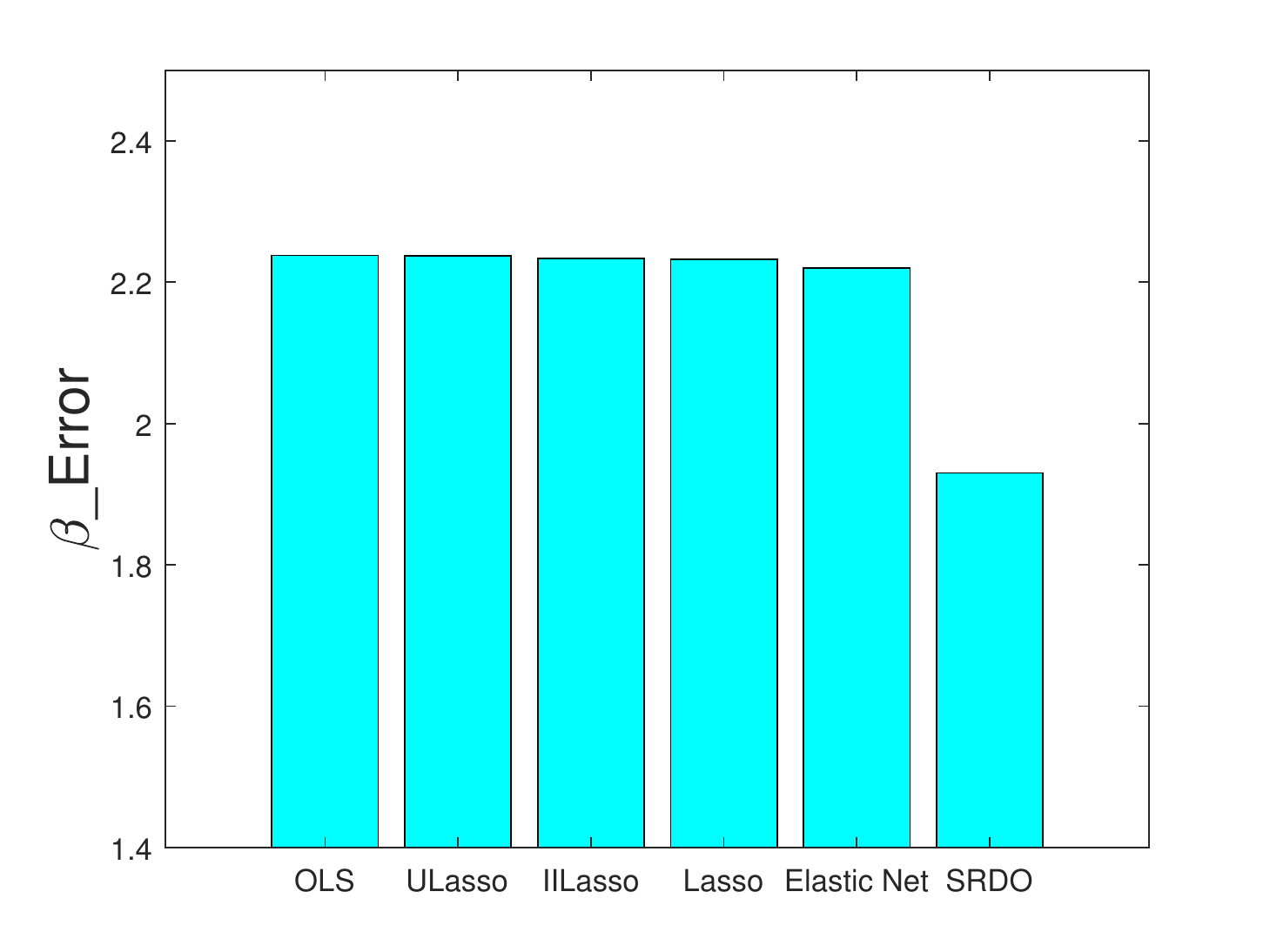}
	}
	\subfloat[Prediction error over different test environments]{
		\includegraphics[width=0.32\linewidth]{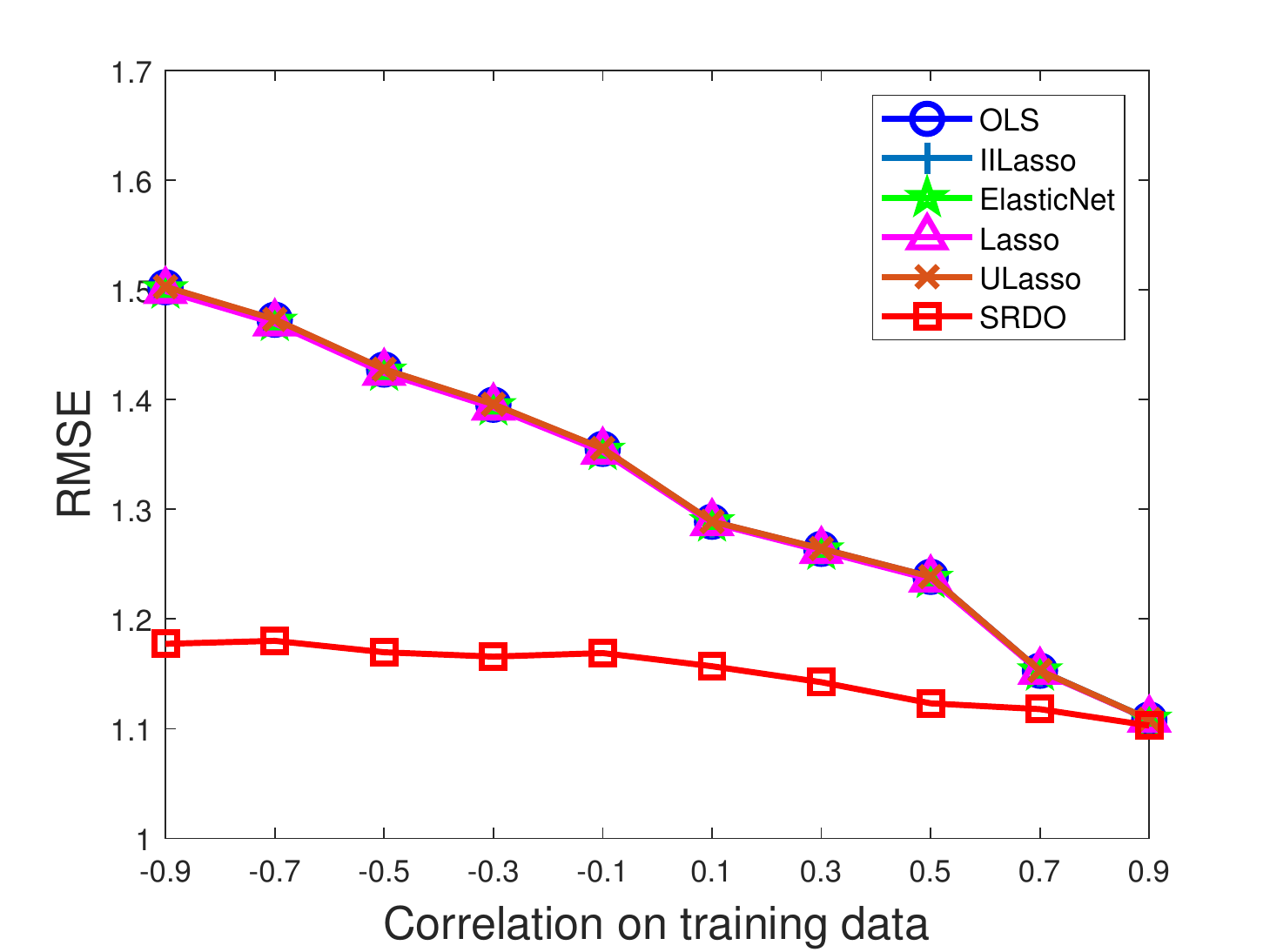}
	}
	\subfloat[Average prediction error\&stability]{
		\includegraphics[width=0.32\linewidth]{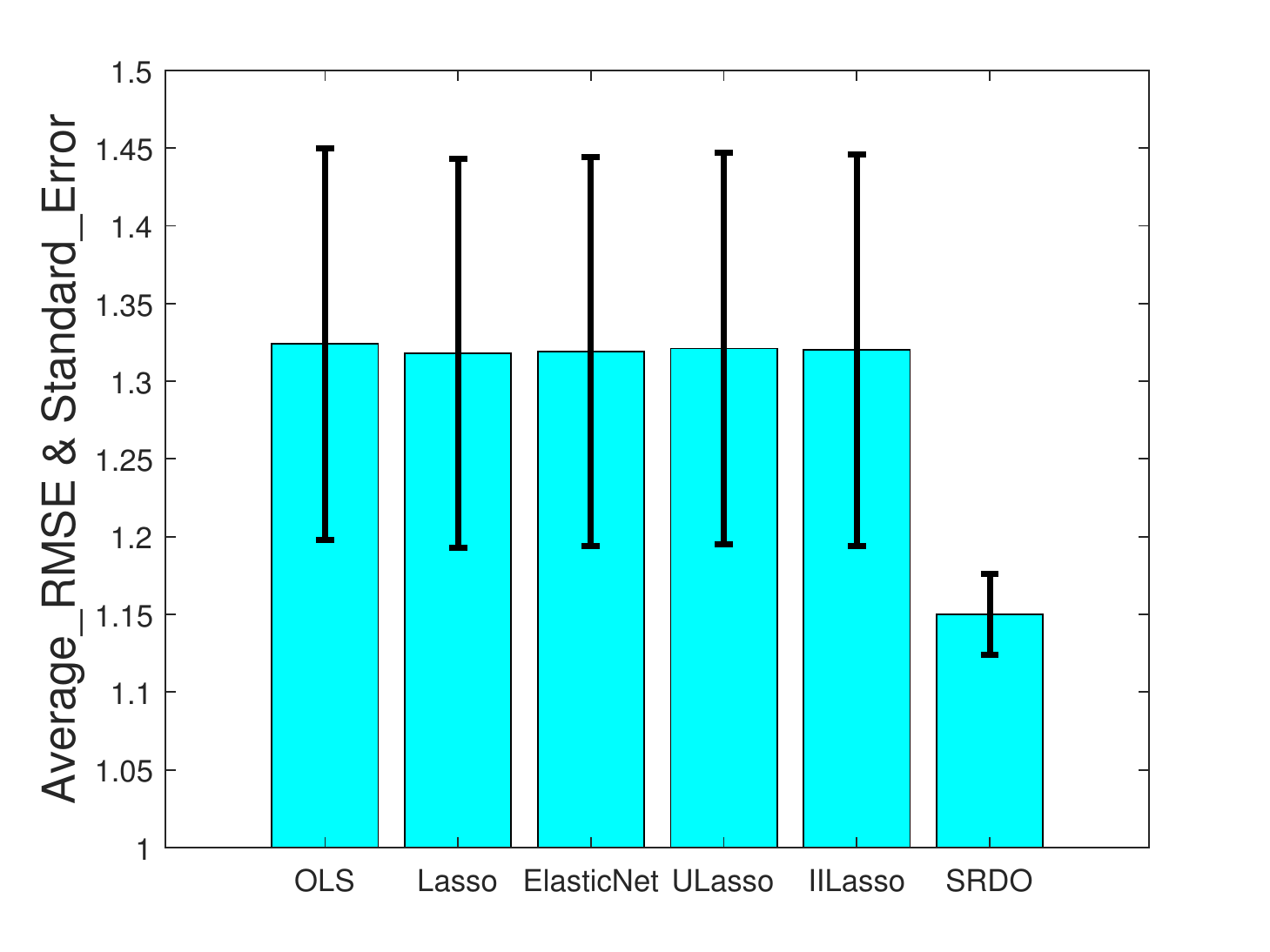}
	}
	\caption{Estimation and prediction results on simulation data generated by $n=1000,\ p=10,\ s=2, \ \rho_{train}=0.8$ and $\overline{\beta}=\{\frac{1}{5},-\frac{2}{5},\frac{3}{5},-\frac{4}{5},1,-\frac{1}{5},\frac{2}{5},-\frac{3}{5},\frac{4}{5},-1\}$.}
	\label{fig:synthetic}
\end{figure*}

\subsection{Simulation Study}

\subsubsection{Experimental Setting}
In simulation study, we generate the design matrix $\mathbf{X}$ from a multivariate normal distribution $\mathbf{X} \sim N(0, \Sigma)$, by specifying the structure of covariance matrix $\Sigma$. We can simulate different correlation structures of $\mathbf{X}$. 
Specifically, we let $\Sigma=\operatorname{Diag}\left(\Sigma^{(1)}, \cdots, \Sigma^{(q)}\right)$ to be a block diagonal matrix whose element $\Sigma^{(l)} \in \mathbb{R}^{s \times s}$ was $\Sigma_{j k}^{(l)}=\rho^{(l)}$ for $j \neq k$ and $\Sigma_{j k}^{(l)}=1$ for $j=k$.
So there will be $q$ groups among all the $p$ variables, and each group contains $s=\frac{p}{q}$ correlated variables. Then we generate bias term $b(\mathbf{X})$ with $b(\mathbf{X})=\mathbf{X}v$, where $v$ is the eigenvector of centered covariance matrix $n^{-1} \sum_{i}\left(x_{i}-\overline{x}\right)\left(x_{i}-\right.\overline{x} )^{\top}$ corresponding to its smallest eigenvalue $\gamma^2$. Finally, we generate the response variable $Y$ as follows:
\begin{eqnarray}
Y=X\overline{\beta}+b(\mathbf{X})+\mathcal{N}(0, 1).
\end{eqnarray}

We evaluate the estimation performance by absolute error ($\beta\_error$) defined as $\|\overline{\beta}-\hat{\beta}\|_1$.
During the training process, we run the experiment for 30 times and report the average $\beta\_error$ as estimation error.
For prediction, we choose the most accurate estimation in training and calculate root mean square error (RMSE) $\frac{1}{n}\sqrt{\sum_{i=1}^{n}(Y_i-\hat{Y}_i)}$, we also carry out this procedure for 30 times  and average the results.
Particularly, in stable learning we want to evaluate the performance of methods in the changing distributions of different test environments.
To do so, we train all the methods with fixed $\rho_{train}$ and generate different test environments by varying the $\rho$ in test data.
Then we report the average prediction error over various test environments to indicate prediction accuracy and its standard deviation to indicate prediction stability.
A stable model is expected to produce not only small average prediction error but also small variance across different test scenarios.

\begin{table*}[htbp]
	\small
	\centering
	\caption{Results of different methods when varying sample size $n$ and correlation $\rho$ of training data.}
	\label{tab:simulation}
	\begin{tabular}{|c|c|c|c|c|c|c|c|c|c|}
		\hline
		\multicolumn{10}{|c|}{\textbf{Scenario 1: varying correlation $\rho$\quad($n=1000,p=10,s=2$)}}\\
		\hline
		$\rho$&\multicolumn{3}{|c|}{$\rho=0.5$}&\multicolumn{3}{|c|}{$\rho=0.7$}&\multicolumn{3}{|c|}{$\rho=0.9$}\\
		\hline
		Methods & $\beta\_error$ & \multicolumn{2}{|c|}{$RMSE$\ ($STD$)} & $\beta\_error$ & \multicolumn{2}{|c|}{$RMSE$\ ($STD$)} & $\beta\_error$ & \multicolumn{2}{|c|}{$RMSE$\ ($STD$)}\\
		\hline % setting 1
		OLS         & 1.528 & \multicolumn{2}{|c|}{1.173(0.047)} & 1.896 & \multicolumn{2}{|c|}{1.261(0.101)} & 2.964 & \multicolumn{2}{|c|}{1.476(0.213)}\\ 
		Lasso       & 1.520 & \multicolumn{2}{|c|}{1.173(0.047)} & 1.892 & \multicolumn{2}{|c|}{1.263(0.102)} & 2.939 & \multicolumn{2}{|c|}{1.484(0.217)}\\
		Elastic Net & 1.515 & \multicolumn{2}{|c|}{1.171(0.046)} & 1.884 & \multicolumn{2}{|c|}{1.263(0.102)} & 2.938 & \multicolumn{2}{|c|}{1.483(0.217)}\\
		ULasso      & 1.527 & \multicolumn{2}{|c|}{1.173(0.047)} & 1.898 & \multicolumn{2}{|c|}{1.260(0.100)} & 2.950 & \multicolumn{2}{|c|}{1.480(0.215)}\\
		IILasso     & 1.534 & \multicolumn{2}{|c|}{1.177(0.049)} & 1.897 & \multicolumn{2}{|c|}{1.260(0.100)} & 2.957 & \multicolumn{2}{|c|}{1.476(0.213)}\\
		Our         & \textbf{1.402} & \multicolumn{2}{|c|}{\textbf{1.141}(\textbf{0.027})} & \textbf{1.759} & \multicolumn{2}{|c|}{\textbf{1.130}(\textbf{0.023})} & \textbf{2.544} & \multicolumn{2}{|c|}{\textbf{1.225}(\textbf{0.065})}\\
		\hline
		\multicolumn{10}{|c|}{\textbf{Scenario 2: varying sample size $n$\quad($p=10,s=2,\rho=0.9$)}}\\
		\hline
		$n$&\multicolumn{3}{|c|}{$n=500$}&\multicolumn{3}{|c|}{$n=2000$}&\multicolumn{3}{|c|}{$n=10000$}\\
		\hline
		Methods & $\beta\_error$ & \multicolumn{2}{|c|}{$RMSE$\ ($STD$)} & $\beta\_error$ & \multicolumn{2}{|c|}{$RMSE$\ ($STD$)} & $\beta\_error$ & \multicolumn{2}{|c|}{$RMSE$\ ($STD$)}\\
		\hline % setting 2
		OLS         & 3.241 & \multicolumn{2}{|c|}{1.382(0.153)} & 3.184 & \multicolumn{2}{|c|}{1.613(0.263)} & 3.168 & \multicolumn{2}{|c|}{1.574(0.243)}\\ 
		Lasso       & 3.232 & \multicolumn{2}{|c|}{1.384(0.154)} & 3.179 & \multicolumn{2}{|c|}{1.600(0.257)} & 3.145 & \multicolumn{2}{|c|}{1.560(0.236)}\\
		Elastic Net & 3.234 & \multicolumn{2}{|c|}{1.383(0.154)} & 3.166 & \multicolumn{2}{|c|}{1.596(0.255)} & 3.137 & \multicolumn{2}{|c|}{1.559(0.235)}\\
		ULasso      & \textbf{3.181} & \multicolumn{2}{|c|}{\textbf{1.382}(0.153)} & 3.182 & \multicolumn{2}{|c|}{1.608(0.260)} & 3.165 & \multicolumn{2}{|c|}{1.577(0.244)}\\
		IILasso     & 3.226 & \multicolumn{2}{|c|}{1.383(0.154)} & 3.184 & \multicolumn{2}{|c|}{1.607(0.260)} & 3.159 & \multicolumn{2}{|c|}{1.575(0.243)}\\
		Our         & 3.421 & \multicolumn{2}{|c|}{1.385(\textbf{0.126})} & \textbf{2.810} & \multicolumn{2}{|c|}{\textbf{1.384}(\textbf{0.150)}} & \textbf{2.762} & \multicolumn{2}{|c|}{\textbf{1.269}(\textbf{0.093})}\\
		\hline
	\end{tabular}
\end{table*}

\subsubsection{Results}
We conduct extensive experiments with different settings on $n$, $p$, $s$, and $\rho_{train}$. Due to the limitation of space, we only report a part of experimental settings and results, and more empirical results could be found in supplementary material.
From figure \ref{fig:synthetic} and Table \ref{tab:simulation}, we have the following observations and analysis:
\begin{itemize}
	\item Ordinary least squares suffers from collinearity in terms of error inflation and yields the worst performance in most of settings, which is consistent with our theoretical analysis.
	\item Lasso does not differentiate itself with OLS much because of the dense $\overline{\beta}$ we used in simulation. 
	The weakest signal has a magnitude of $0.2$ which is comparable to the largest one, so it is typically hard for coefficients shrinkage mechanism to work in such setting.
	\item Elastic Net performs slightly better than the other baselines due to its involvement of $l_2$ regularization in the collinear case, which has been discussed in \cite{Tibshirani1996Regression,zou2005regularization}.
	\item ULasso and IILasso can not quite solve the problem of collinearity in this experiment because they assume a sparse structure within correlated groups (i.e. there exists only one active variables among several correlated variables), which is not satisfied here.
	\item From Figure \ref{fig:synthetic}, we can find SRDO achieves smallest estimation error under strong correlations between variables, and a more stable prediction performance in different test settings, which achieves the goal of stable learning.
	Note that in the right end of Figure \ref{fig:synthetic} (b), all the methods generate comparable results, which coincides with I.I.D. assumption in that the strong collinearity in training data still persists in test data.
	However, as the discrepancy of training and test distribution getting larger (from the right end to the left end), the performance of baselines deteriorate rapidly.
	\item From Table \ref{tab:simulation}, we can find that when collinearity in training data becomes stronger, our method gains more improvement over baselines in all aspects including estimation error, prediction error and prediction stability.
	We also notice that our method is generally affected by sample size $n$.
	It typically performs well in large data, in relatively small sample setting, however, SRDO may suffer from variance inflation in terms of parameter estimation, which counteracts the benefit brought by bias reduction.
\end{itemize}
These results demonstrate the superior capability of our method in handling the negative effects aroused by strong collinearity among variables.

\begin{figure}[h]
	\centering
	\includegraphics[width=0.95\columnwidth]{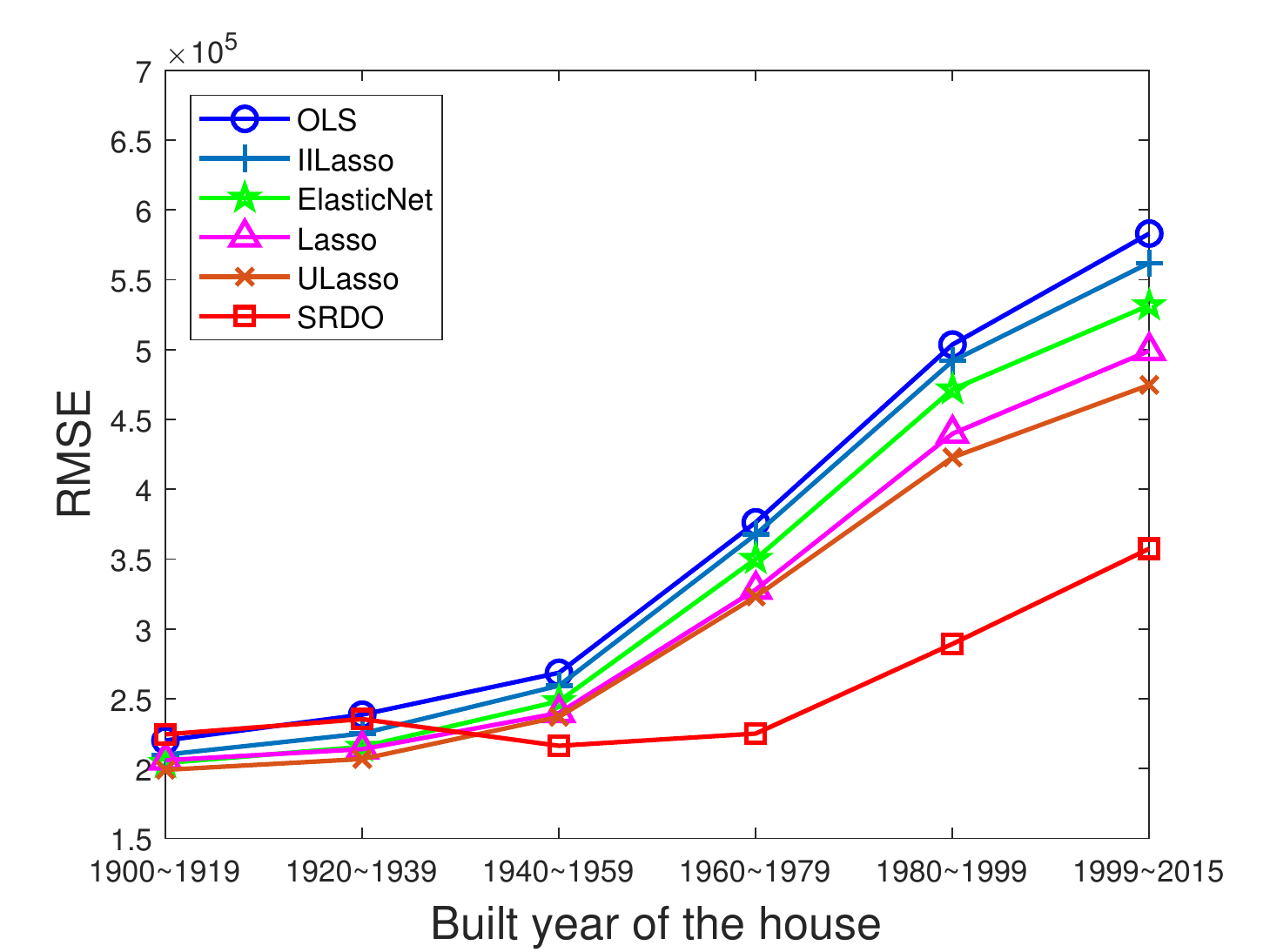}
	\caption{Prediction performances over various built periods of house. All the models are trained on the first period $built\_year \in [1900, 1919]$ and tested on all the six periods.}
	\label{fig:House_RMSE}
\end{figure}

\subsection{Real World Regression Experiments}

\subsubsection{Dataset and Experimental Setting}
In this experiment, we use a real world regression dataset (Kaggle) of house sales prices from King County, USA, which includes the houses sold between May 2014 and May 2015 .
The outcome variable is the transaction price of the house and each sample contains 16 predictive variables such as the built year of the house, number of bedrooms, number of bathrooms, and square footage of home etc.
We normalize all the predictive variables to get rid of the influence by their original scales.

To test the stability of different algorithms, we simulate different "environments" according to the built year of the house.
It is fairly reasonable to assume the distribution of predictors as well as their collinearity may vary along the time, due to the changing popular style of architectures.
Specifically, the houses in this dataset were built between 1900$\sim$2015 and we split the dataset into 6 periods, where each period approximately covers a time span of two decades.
We train all the methods on the first period where $built\_year \in [1900, 1919]$ with cross validation, and test them on all the six periods respectively.

\begin{figure}[h]
	\centering
	\includegraphics[width=0.95\columnwidth]{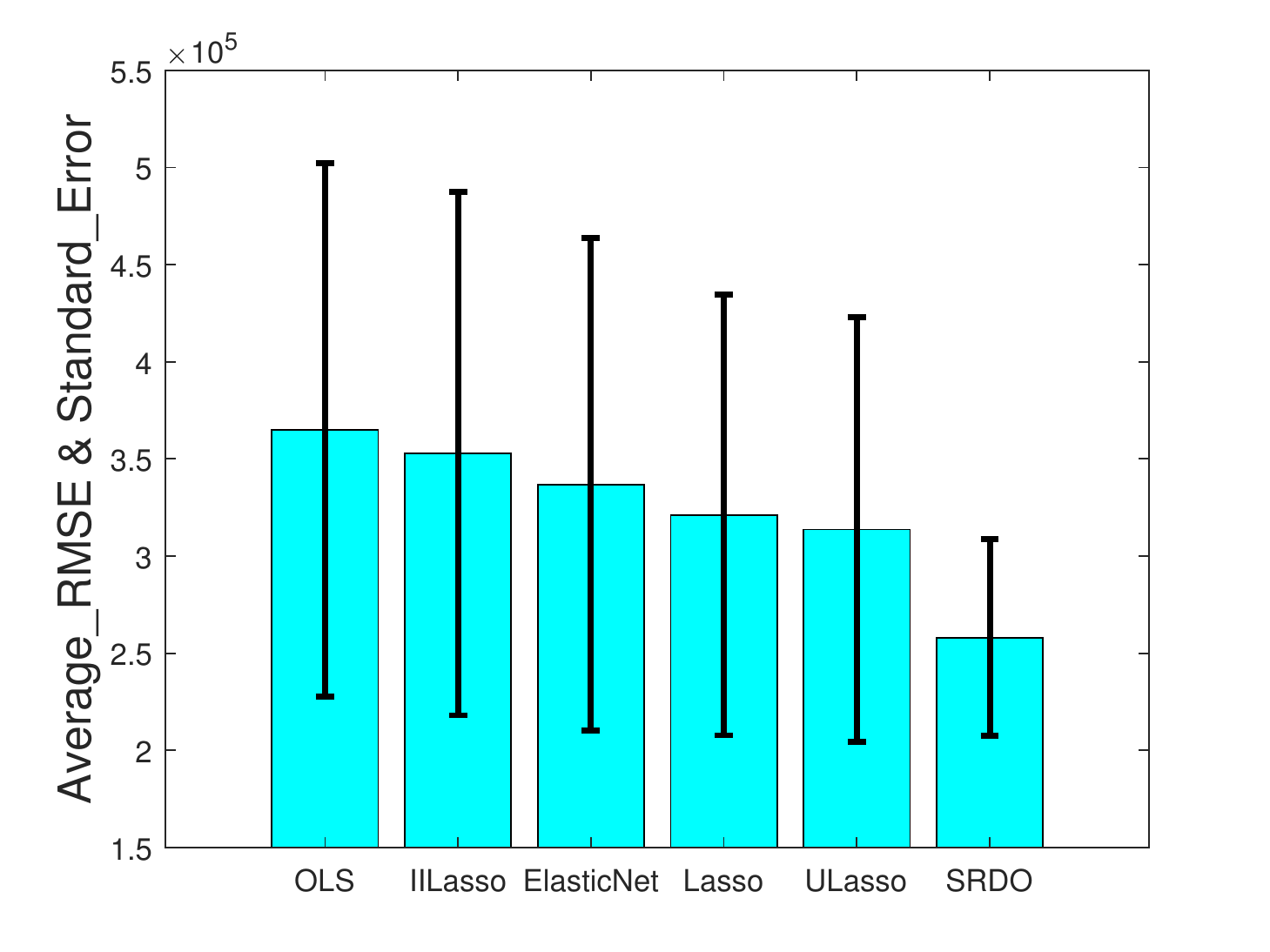}
	\caption{Average performance and stability over various built periods of house.}
	\label{fig:House_Stability}
\end{figure}

%\begin{figure}[h]
%	\centering
%	\includegraphics[width=0.95\columnwidth]{figure/Ads_AUC.eps}
%	\caption{Classification performance across various users' age groups. 
%		All the models are trained on $Age \in [20, 30)$ and tested on all the five groups.}
%	\label{fig:Ads_AUC}
%\end{figure}
%
%\begin{figure}[h]
%	\centering
%	\caption{Average AUC of all the environments and stability.}
%	\includegraphics[width=0.87\columnwidth]{figure/Ads_stability.eps}
%	\label{fig:Ads_Stability}
%\end{figure}

\begin{figure*}[ht]
	\centering
	\subfloat[AUC over different test environments.]{
		\includegraphics[width=0.9\columnwidth]{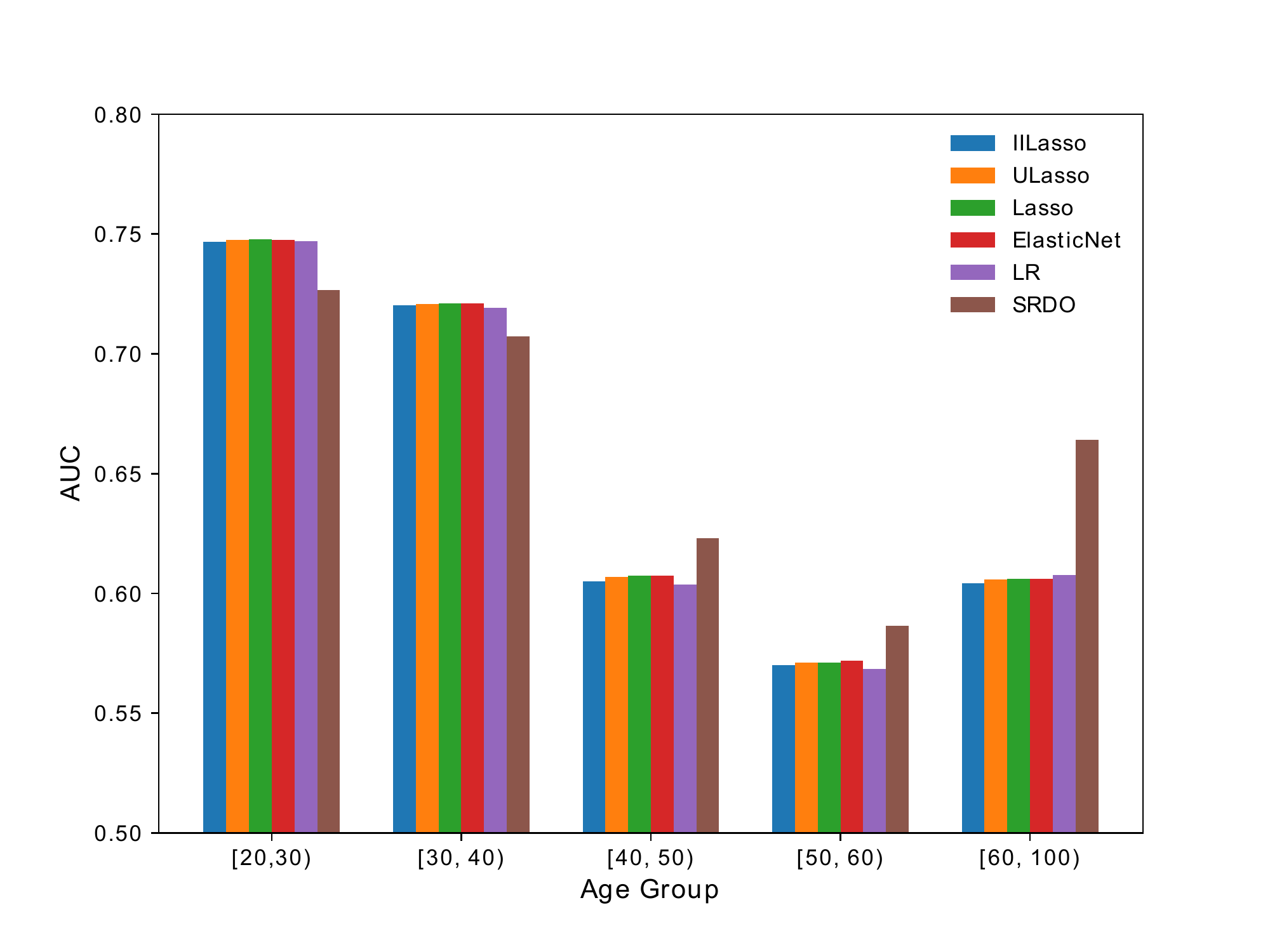}
	}
	\subfloat[Average AUC of all the environments and stability.]{
		\includegraphics[width=0.87\columnwidth]{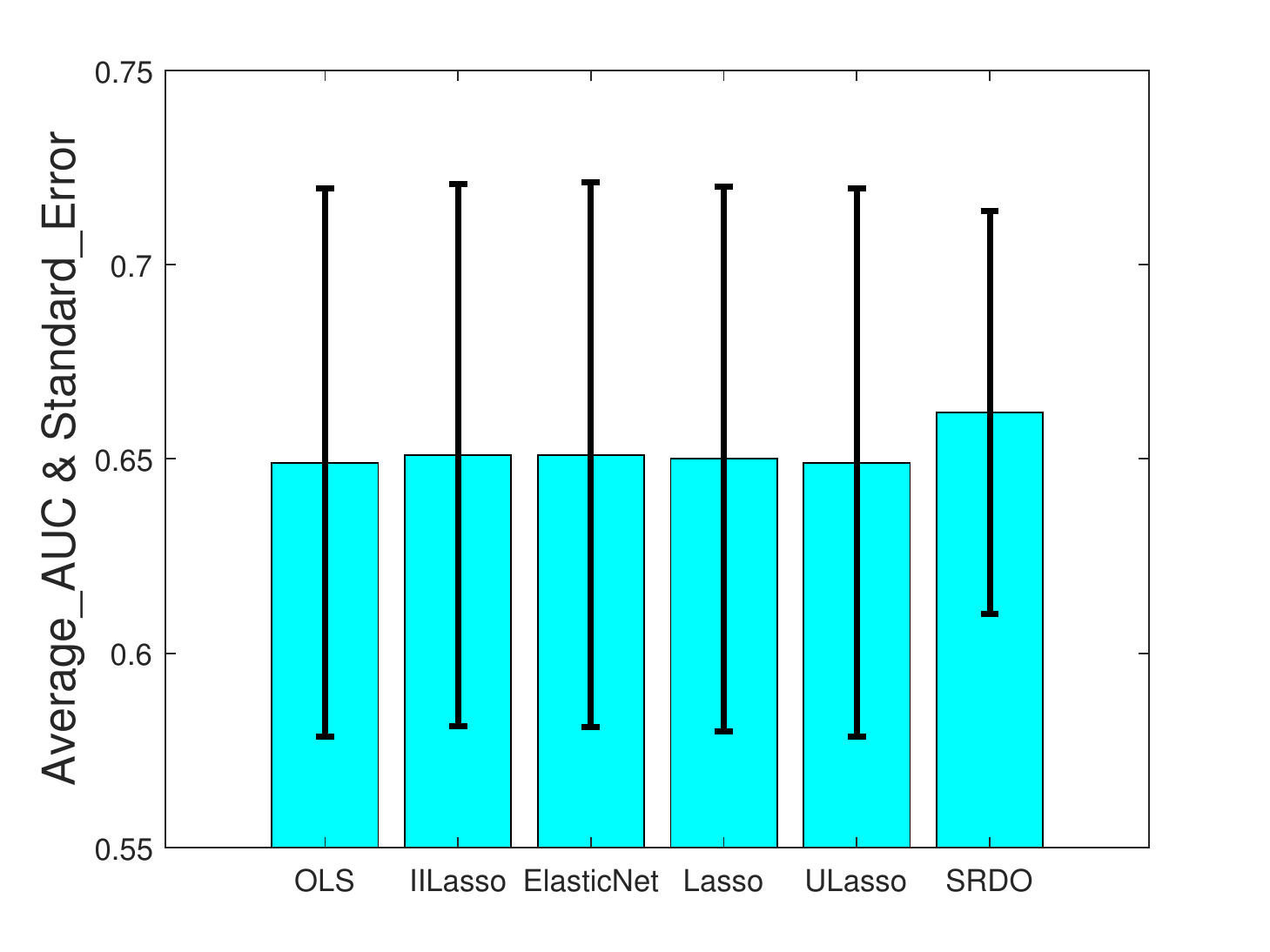}
	}
	\caption{Classification performance across various users' age groups. 
		All the models are trained on $Age \in [20, 30)$ and tested on all the five groups.}
	\label{fig:Ads_Result}
\end{figure*}

\subsubsection{Results}
From Figure \ref{fig:House_RMSE} we can find that our method achieves not only the smallest average error but also a better stability over different test environments compared with other baselines.
The results coincide with our assumption that the data gathered through a long time span would undergo changes in collinearity patterns.
So the removal of collinearity yields a more stable model of changing distributions.
OLS performs the worst due to its nature of sensitivity to collinearity which leads to large estimation variance.
Meanwhile, ElasticNet and Lasso gain a notable margin against OLS as they involve the $l_1$ and $l_2$ regularization for sparseness and variance reduction, which is usually favorable in real applications.
Note that ULasso and IILasso, report different performances compared with Lasso.
A plausible reason is that IILasso impose stronger penalty on the correlation between variables than ULasso by using $R_{j k}=\left|r_{j k}\right| /\left(1-\left|r_{j k}\right|\right)$ instead of $R_{j k}=r_{j k}^{2}$ in ULasso, where $r_{j k}=\frac{1}{n}\left|X_{j}^{\top} X_{k}\right|$ is the absolute sample correlation. 
And the over penalty results in the over-sparsity of selected models.

From Figure \ref{fig:House_Stability}, we can find a clear error inflation along the time axis for all the methods.
Note that the models are trained in period 1. 
The longer time interval from period 1, the larger distribution shifting may incur, meaning more challenging prediction tasks. 
The results show that our method performs much better than baselines in period 3-6, and also produce comparable performances with baselines in the first two periods without obvious distribution change. 
Therefore, in practical use, our algorithm is more reliable, especially when one expects to encounter obvious environment changes in test scenarios. 

\subsection{Real World Classification Experiments}

\subsubsection{Dataset and Experimental Setting}
\textbf{WeChat Ads} is an online advertising dataset collected from Tencecnt WeChat App during September 2015 which contains the user feedback over advertisement flow.
For each advertisement, there are two types of feedbacks: "Like" and "Dislike".  
For each user, there are 56 features characterizing his/her profile including (1) demographic attributes, such as age and gender, (2) number of friends, (3) device (iOS or Android), and (4) the user's various custom settings on WeChat App.

To test the stability of different algorithms, we simulate different environments via stratification over user's age since we consider age as a vital factor which may affect one's personal interest, online behavior etc.
Specifically, we split the dataset into 5 subsets by user's age, including $Age \in [20, 30)$, $Age \in [30, 40)$, $Age \in [40, 50)$, $Age \in [50, 60)$ and $Age \in [60, 100)$.
We train all the methods on users with $Age \in [20, 30)$ via cross validation, and test them on all the five age groups respectively.

\subsubsection{Results}

We plot the classification performance in terms of AUC for each method in Figure \ref{fig:Ads_Result}.
We can find that generally the performance of all the methods would degrade when tested on people from different groups, which is fairly reasonable in that the online behavior patterns are considerably different for people with different age.
Similar to the previous regression experiments, our method generally helps when the distribution shifting is large and more robust to the discrepancy between training and test distribution.
One plausible reason why overall improvement of AUC is moderate compared with regression task is that the collinearity problem of this dataset is not as severe as the hosue sales data, which incurs less inflation of estimation bias for traditional methods.

\section{Conclusion and Discussion}
In this paper, we investigated the stable learning problem for linear regression with model misspecification bias.
We proposed a method to reduce the effect of collinearity in the training data via sampling reweighting. 
We theoretically showed that there exists an optimal set of sample weights that can make the design matrix nearly orthogonal in idealized situations. 
In more realistic situations, the empirical results show that our method can improve the stability of linear models when the test data differs from the training data. 
Our method is a general data pretreatment method, which can be seamlessly integrated into classical linear models such as ordinary least squares and logistic regression.
It provides a unified approach to alleviate the problem of collinearity for statistical estimation.

\section{Acknowledgements}
This work was supported in part by National Key R\&D Program of China (No. 2018AAA0102004, No. 2018AAA0101900), National Natural Science Foundation of China (No. 61772304, No. 61521002, No. 61531006, No. U1611461), Beijing Academy of Artificial Intelligence (BAAI).
Peng Cui, Tong Zhang and Kun Kuang are co-corresponding authors.
All opinions in this paper are those of the authors and do not necessarily reflect the views of the funding agencies.

\bibliographystyle{aaai}
\bibliography{aaai.bib}
\end{document}